\documentclass{article}

\usepackage{iclr2027_conference,times}

\usepackage{amsmath,amsfonts,bm}

\def\eqref#1{equation~\ref{#1}}

\def\1{\bm{1}}

\DeclareMathAlphabet{\mathsfit}{\encodingdefault}{\sfdefault}{m}{sl}
\SetMathAlphabet{\mathsfit}{bold}{\encodingdefault}{\sfdefault}{bx}{n}

\DeclareMathOperator*{\argmax}{arg\,max}

\usepackage{hyperref}
\usepackage{url}
\usepackage{amsmath,amssymb,amsthm}
\usepackage{booktabs}
\usepackage{graphicx}
\usepackage{subcaption}
\usepackage{multirow}
\usepackage{makecell}
\usepackage{xcolor}
\usepackage{listings}
\usepackage{tcolorbox}
\usepackage{enumitem}

\definecolor{BrickRed}{rgb}{0.6,0,0}
\definecolor{RoyalBlue}{rgb}{0,0,0.5}
\definecolor{Tdgreen}{rgb}{0,0.4,0.7}

\definecolor{CodeKeyword}{HTML}{008000}
\definecolor{CodeComment}{HTML}{287D80}
\definecolor{CodeFunction}{HTML}{0000CC}
\definecolor{CodeBackground}{HTML}{F2F2F2}
\lstdefinestyle{tisd-pseudocode}{
  language=Python,
  basicstyle=\small\ttfamily,
  keywordstyle=\color{CodeKeyword}\bfseries,
  commentstyle=\color{CodeComment},
  stringstyle=\color{BrickRed},
  emph={tisd_step},
  emphstyle=\color{CodeFunction},
  numbers=left,
  numberstyle=\scriptsize\color{black!60},
  numbersep=8pt,
  stepnumber=1,
  firstnumber=1,
  numberblanklines=true,
  columns=fullflexible,
  keepspaces=true,
  showstringspaces=false,
  tabsize=4,
  breaklines=true,
  aboveskip=0pt,
  belowskip=0pt
}

\hypersetup{
  colorlinks=true,
  allcolors=black,
  linkcolor=BrickRed,
  citecolor=RoyalBlue,
  pdftitle={TISD: On-Policy Self-Distillation with Trajectory Intervention},
  pdfauthor={Taeckyung Lee, Rinat Amankos, Jeonghye Kim, Hyungjun Yoon, Woogyeol Jin, Sung-Ju Lee}
}

\theoremstyle{plain}
\newtheorem{proposition}{Proposition}

\newcommand{\method}{\text{TISD}}
\newcommand{\Lorig}{\mathcal{L}_{\tt{SDPO}}}

\newcommand{\sg}{{\tt sg}}
\newcommand{\tstar}{t^{*}}

\title{\method{}: On-Policy Self-Distillation with\\Trajectory Intervention}
\author{
\begin{minipage}[t]{\textwidth}
\vspace{-0.2in}
\raggedright
\textbf{Taeckyung Lee$^1$}\hspace{0.6cm}\textbf{Rinat Amankos$^1$}\hspace{0.6cm}\textbf{Jeonghye Kim$^1$}\hspace{0.6cm}\textbf{Hyungjun Yoon$^1$}\\
\textbf{Woogyeol Jin$^1$}\hspace{0.6cm}\textbf{Sung-Ju Lee$^1$}\\[3pt]
{\normalfont
\textsuperscript{1}KAIST \\[1pt]
{\small \texttt{\{taeckyung, profsj\}@kaist.ac.kr}} \\[2pt]
\bloglogo ~ {\small \href{https://miil.kaist.ac.kr/projects/tisd/}{\texttt{project page}}}
~~
\githublogo ~ {\small \href{https://github.com/taeckyung/TISD}{\texttt{taeckyung/TISD}}}
}
\end{minipage}
}
\date{}

\def\githublogo{\raisebox{-1.5pt}{\includegraphics[height=1.05em]{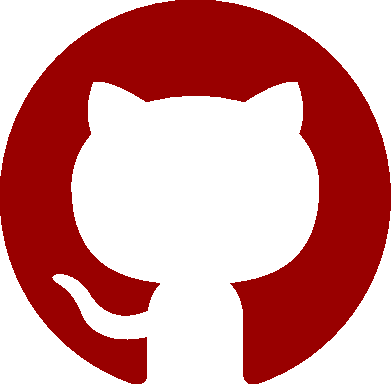}}}
\def\bloglogo{\raisebox{-1.5pt}{\includegraphics[height=1.05em]{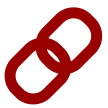}}}

\begin{document}

\iclrfinalcopy
\maketitle
\lhead{Preprint}
\vspace{-0.16in}

\begin{abstract}

On-policy self-distillation (OPSD) provides dense teacher targets, but evaluates them only along student-sampled rollouts. When the privileged teacher favors an alternative action at a visited prefix, OPSD can provide a target for the branch decision but cannot supervise the successor contexts induced by that action unless the student samples it. This creates a training-time data-collection bottleneck and suggests a different role for teacher--student disagreement: proposing a trajectory branch rather than identifying a sufficient local repair. Our diagnostic framework using controlled token interventions reveals that a teacher-preferred token at peak disagreement can improve student continuation success, while its local corrective value is limited. Motivated by this finding, we introduce a simple branch--regenerate--distill algorithm, \textbf{Trajectory-Intervention Self-Distillation (\method{})}. \method{} forces a teacher-selected branch action, returns suffix generation to the student, and distills the full trajectory under the privileged-context-conditioned teacher. Across the coding models, \method{} improves average Avg@4 over SDPO by 1.2 percentage points. Across the science domains, it improves average Avg@128 by 0.8 points under an equal-step budget and by 0.3 points under an equal-time budget. These results support teacher-guided branching as a way to expose useful successor contexts for self-distillation.

\end{abstract}

\section{Introduction}
\label{sec:intro}

Reinforcement learning with verifiable rewards (RLVR) trains large language models on tasks whose outputs can be automatically checked, such as code against test suites or mathematical reasoning against ground-truth answers \citep{le2022coderl,shao2024deepseekmath,guo2025deepseekr1}. The verifier provides a response-level reward to an output that may span hundreds or thousands of tokens.

On-policy distillation (OPD) provides denser supervision by training the student on its own rollouts using token-level targets from a teacher model~\citep{agarwal2024opd}. On-policy self-distillation (OPSD) removes the need for a distinct teacher: the {teacher} is the same model with additional privileged context unavailable to the student, such as a code execution trace, correct rollout, or ground-truth solution~\citep{hubotter2026sdpo,zhao2026sdr,song2026rltf}. The student then minimizes the divergence from the teacher at each token.

In standard OPSD, teacher targets are evaluated at prefixes in student-sampled trajectories. When the teacher favors a different token at a given prefix, its target can teach the student to prefer that token, while the collected suffix follows the original student choice. Supervising the continuation \textit{after} the teacher-preferred alternative requires collecting its successor contexts. This creates a training-time data-collection bottleneck: standard OPSD provides direct teacher supervision at subsequent prefixes on an alternative branch only when those prefixes appear in student-sampled rollouts.

Teacher-generated trajectories offer another way to expose alternative continuations. TRD~\citep{jiang2026trajectory} uses a privileged-context teacher to rewrite student responses. The student is then trained on continuations generated by a teacher with access to information unavailable to the student at deployment. In SciKnowEval, TRD exhibits late-training accuracy collapse as responses reach the generation limit (Section~\ref{sec:method-student-generation}). This failure case motivates using the teacher to select a branch action and the student to generate the continuation.

Our insight is to use teacher--student disagreement to identify where to branch and collect an alternative student-generated continuation for supervision. We first analyze where disagreement concentrates, then test its value for trajectory steering through controlled token interventions (Section~\ref{sec:analysis}). On failed LiveCodeBench rollouts, the peak-disagreement position $\tstar$ accounts for $9$--$20\%$ of total response KL divergence across model scales, identifying a concentrated target for intervention. 
On these failed rollouts, Qwen3-8B achieves $9.9\%$ continuation success when the teacher-preferred token is forced at $\tstar$, compared with $1.7\%$ when the original token is retained, with the student regenerating the suffix in both cases.
This result motivates using a single teacher-selected action to collect alternative student continuations for dense supervision within the same training iteration.

We turn this insight into \textbf{Trajectory-Intervention Self-Distillation~(\method{})}, a branch--regenerate--distill framework (Figure~\ref{fig:overview}). For each eligible failed rollout, \method{} finds the peak-divergence position $\tstar$, forces the teacher's preferred token, and lets the student regenerate the suffix. It then distills teacher predictions along the full trajectory, including the original student prefix, the intervened token, and the regenerated student suffix. This supervises both the branch decision and subsequent predictions at the newly visited prefixes.

{\method{} improves Avg@4 over SDPO by 1.2 percentage points on average across three coding models. Across the four science domains, \method{} improves Avg@128 over SDPO by an average of 0.8 points under equal training steps and 0.3 points under equal training time (Section~\ref{sec:experiments}). These results demonstrate the benefit of teacher-guided branching both with and without rich execution feedback.}

Our contributions are:
\begin{itemize}[itemsep=0pt, topsep=0pt, parsep=0pt, partopsep=0pt, leftmargin=*]
\item A diagnostic framework that distinguishes local correction from trajectory steering through controlled token interventions. We show that direct token substitution rarely repairs failed responses, whereas teacher-guided branching at peak disagreement improves Qwen3-8B continuation success (Section~\ref{sec:analysis}).
\item \textbf{\method{}}, a simple self-distillation method that forces a teacher-preferred token, lets the student regenerate the suffix, and distills teacher predictions along the resulting trajectory. This makes teacher supervision available along alternative continuations without waiting for the student to select the branch on its own (Section~\ref{sec:method}).
\item {An evaluation across coding and science showing higher average performance than SDPO under the reported step- and time-based comparisons (Section~\ref{sec:experiments}).}

\end{itemize}

\begin{figure*}[t]
  \centering
  \begin{minipage}[t]{\textwidth}
    \centering
    \includegraphics[width=\linewidth]{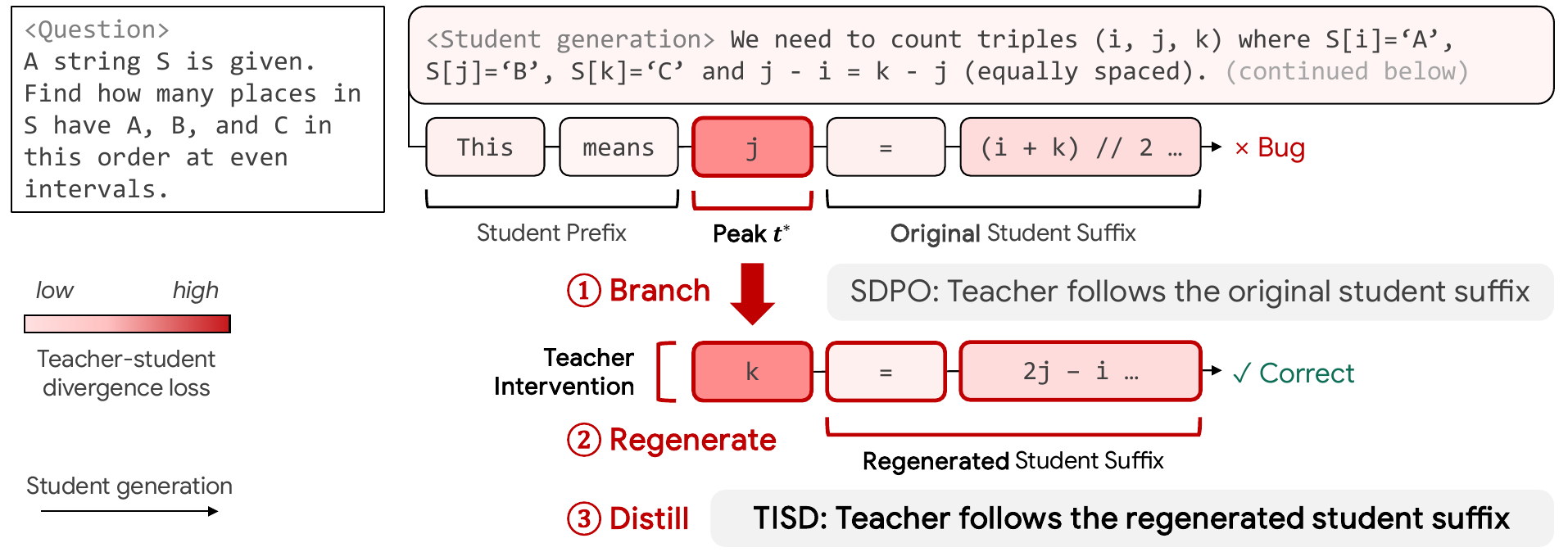}

\captionsetup{skip=3pt}
\captionof{figure}{\method{} changes where self-distillation is evaluated. Standard OPSD (e.g., SDPO) distills teacher targets along the original student trajectory, including the suffix that follows the student's sampled action. \method{} instead forces the teacher-preferred token at the peak-divergence position, lets the student regenerate the suffix, and distills the resulting trajectory. This exposes successor prefixes that are absent from the original rollout. }
    \label{fig:overview}
  \end{minipage}
  \vspace{-0.2cm}
\end{figure*}

\section{Related Work}
\label{sec:related}

\paragraph{On-policy (self-)distillation.}
On-policy self-distillation (OPSD) derives dense targets from a student-derived teacher conditioned on privileged information. SDPO~\citep{hubotter2026sdpo} and the Self-Distilled Reasoner~\citep{zhao2026sdr} use successful rollouts or environment feedback to construct teacher targets, while RLTF~\citep{song2026rltf} formalizes learning from text feedback and analyzes its optimization behavior. 
\citet{kim2026why} further identify suppression of uncertainty expression as a source of degraded out-of-distribution reasoning in self-distillation.
Related OPD methods modify how teacher guidance is provided along student trajectories: TGPO~\citep{liu2026teacherguidedpolicyoptimizationonpolicy} trains the student to predict the teacher's argmax token at each visited prefix, while Trajectory-aware OPD~\citep{jiang2026topd} uses near-future trajectory information to identify divergent states and distribute guidance across subsequent tokens. TRD~\citep{jiang2026trajectory} changes the training trajectory itself by using a privileged-context teacher to rewrite the student response before distillation. \method{} uses teacher guidance to change the trajectory through a single branch action, then returns generation to the student. The student regenerates the suffix without privileged context, and the privileged-context teacher provides dense targets along the resulting trajectory.

\vspace{-0.1cm}\paragraph{Credit assignment and selective token signals.} Prior work improves learning from response-level rewards through process reward models~\citep{lightman2024verify,wang2024mathshepherd}, Monte Carlo value estimation~\citep{kazemnejad2024vineppo}, return decomposition~\citep{parthasarathi2025grpolambda}, implicit process rewards~\citep{cui2025prime}, execution-derived penalties~\citep{liu2023rltf}, and revision from execution feedback~\citep{gehring2025rlef}. Other studies show that token-level signals are highly non-uniform: low-probability tokens can dominate policy-gradient updates~\citep{yang2025lowprob}, and high-entropy minority tokens disproportionately drive RL gains~\citep{wang2025beyond8020}. Selective distillation methods use token-importance signals to select or weight supervision~\citep{xu2026tip,tavor2026rethinking,huang2026selectkd}. 
\method{} instead uses teacher–student divergence to select a branch point and forces the teacher-preferred token. The student then regenerates the suffix, exposing new prefixes for dense teacher supervision.

\section{Teacher-Student Divergence as a Trajectory Steering Signal}
\label{sec:analysis}

Standard OPSD evaluates teacher targets at prefixes in student-sampled rollouts. When the teacher favors an alternative token, the collected rollout still follows the student's original choice and does not contain the continuation induced by that alternative. We investigate three questions:
\begin{enumerate}[leftmargin=2em]
    \item \textbf{Where is teacher--student disagreement concentrated?} (\S\ref{sec:analysis-setup})
    \item \textbf{Does strong teacher--student disagreement identify a local correction sufficient to repair the response?} (\S\ref{sec:analysis-findings})
    \item \textbf{Does taking the teacher-preferred action expose a better student continuation?} (\S\ref{sec:analysis-steering})
\end{enumerate}

We conduct the main diagnostics on the Qwen3 series~\citep{yang2025qwen3} in the SDPO~\citep{hubotter2026sdpo} setting on LiveCodeBench v6~\citep{jain2024livecodebench}. Appendix~\ref{app:analysis} provides the intervention protocol and additional analyses.

\subsection{Where Is Teacher--Student Disagreement Concentrated?}
\label{sec:analysis-setup}

We analyze how teacher--student divergence is distributed across response positions to identify candidate intervention points.

OPSD derives a teacher distribution $\pi_\phi$ from the student $\pi_\theta$ and conditions it on privileged context $c$ unavailable to the student during generation. For a student rollout $y=(y_1,\dots,y_T)$ generated from prompt $x$, the per-rollout SDPO distillation loss is
\begin{equation}
\label{eq:sdpo}
\Lorig
= \frac{1}{T}\sum_{t=1}^{T}
\mathcal D_t \!\left(
\pi_\theta(\cdot \mid x, y_{<t})
\;,\;
\sg[\pi_\phi(\cdot \mid x, y_{<t}, c)]
\right),
\end{equation}
where $\mathcal D_t $ denotes the per-token divergence metric and $\sg[\cdot]$ denotes stop-gradient. The privileged context includes a successful sibling from the same rollout group when available. In the rich-feedback setting, execution feedback, such as error traces, is used when no successful sibling is available~\citep{hubotter2026sdpo}.
For the diagnostic, we measure teacher--student disagreement at each position using reverse KL, $\mathcal D_t =\mathcal D_{\tt KL}\!\left(\pi_\theta(\cdot\mid x,y_{<t})\,\|\,\pi_\phi(\cdot\mid x,y_{<t},c)\right)$, a default configuration of SDPO on LiveCodeBench. 

\vspace{-0.1cm}\paragraph{Why analyze failed rollouts?} Failed rollouts provide useful supervision and allow us to test whether teacher guidance corrects an unsuccessful response. At training step 80, Qwen3-1.7B achieves 48.0\% validation accuracy when SDPO supervises both successful and failed rollouts, compared with 37.2\% when supervision is restricted to successful rollouts (Appendix~\ref{app:failed-supervision}). This result motivates examining the guidance available on failed rollouts and its value for correction and trajectory steering.

\vspace{-0.15cm}\paragraph{Concentrated teacher--student disagreement.}
We observe that teacher--student disagreement is highly uneven across the response positions. The \textbf{peak-divergence position} $\tstar=\argmax_t D_t$ alone carries $9$--$20\%$ of total response KL across model scales (Figure~\ref{fig:kl-distribution}), a $40$--$80\times$ concentration relative to a mean KL-divergence. This concentration motivates examining $\tstar$ as a candidate intervention point. We assess its local corrective value by examining its semantic role and testing whether substituting the teacher-preferred token repairs the original response.

\begin{figure*}[t]
  \centering

  \iffalse
  \begin{minipage}[t]{0.48\textwidth}
    \centering
    \includegraphics[width=\linewidth]{figures/fig_tstar_semantic_role.pdf}
    \captionsetup{skip=3pt}
    \caption{Semantic-role distribution (\%) of the student token at $\tstar$ across Qwen3 series (SDPO step 0, LiveCodeBench v6). Evaluated by LLM judge (GPT-5.4).}
    \label{fig:tstar-semantic-role}
  \end{minipage}
  \fi
  \begin{minipage}[t]{0.66\textwidth}
    \vspace{0pt}
    \centering
    \begin{subfigure}[t]{0.484848\linewidth}
      \centering
      \includegraphics[width=\linewidth,trim=1pt 0 1pt 0,clip]{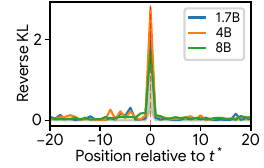}
      \caption{Mean per-token KL profile around $\tstar$.}
      \label{fig:kl-profile}
    \end{subfigure}
    \hfill
    \begin{subfigure}[t]{0.484848\linewidth}
      \centering
      \includegraphics[width=\linewidth,trim=1pt 0 1pt 0,clip]{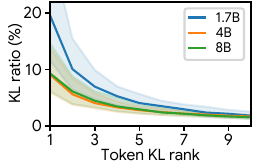}
      \caption{KL concentration ratio by token rank.}
      \label{fig:kl-conc-ratio}
    \end{subfigure}
    \captionsetup{skip=3pt}
    \caption{Concentration of token-level SDPO loss (reverse KL) on failed LiveCodeBench v6 rollouts (Qwen3). \textbf{(a)} Mean per-token KL in a $\pm20$-token window around the peak-divergence position $\tstar$. \textbf{(b)} Fraction of total response KL carried by tokens ranked by KL. The peak token is sharp and consistent across scales, carrying $9$--$20\%$ of total response KL across scales, while the rank-10 contribution is below $1\%$.}
    \label{fig:kl-distribution}
  \end{minipage}
  \hfill
  \begin{minipage}[t]{0.32\textwidth}
    \vspace{0pt}
    \centering
    \includegraphics[width=\linewidth,trim=1pt 0 1pt 0,clip]{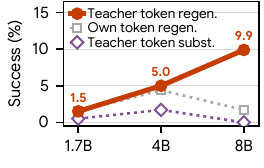}
    \vspace{-4.5pt}
    \captionsetup{skip=0pt}
    \caption{Execution success after intervention at $\tstar$ on failed LiveCodeBench v6 rollouts (Qwen3). \emph{Teacher/own token} force the corresponding action at $\tstar$ and then regenerate the suffix; \emph{substitution} updates only the token at $\tstar$ and preserves the original suffix. }
    \label{fig:tstar-continuation-success}
  \end{minipage}
  \vspace{-0.2cm}
\end{figure*}

\subsection{Peak Divergence Rarely Provides a Local Repair}
\label{sec:analysis-findings}

\paragraph{Peak divergence often occurs outside direct execution decisions.}
On failed rollouts, the semantic role of $t^*$ shifts toward non-code text as model scale increases (Appendix Table~\ref{tab:semantic-role}). Non-code tokens, including comments, docstrings, code fences, and reasoning prose, account for 45.0\% of peak positions at 1.7B and over 70\% at 4B and 8B, while algorithmic tokens decrease from 36.1\% to 11.6\%. Thus, particularly at larger scales, the position of maximal teacher--student disagreement often does not correspond to an explicit algorithmic or control-flow decision. This suggests that peak divergence does not necessarily identify the token directly responsible for execution failure.

\vspace{-0.1cm}\paragraph{Teacher preference alone is rarely a sufficient local repair.}
Directly substituting the teacher-preferred token into the original failed response and rerunning the test suite repairs only $0.0$--$1.7\%$ of failures across model scales (Figure~\ref{fig:tstar-continuation-success}). In contrast, allowing the privileged teacher to generate the suffix from the same intervention point increases success to $7.9\%$ (1.7B), $10.5\%$ (4B), and $13.3\%$ (8B). These results indicate that the teacher-preferred token is rarely sufficient to repair the original trajectory on its own, while changing the continuation after that token can recover a subset of failed responses.

\subsection{A Teacher-Selected Action Can Redirect the Student}
\label{sec:analysis-steering}

\paragraph{A single-token intervention can redirect student generation.}
We force either the teacher's argmax token or the student's original token at $t^*$, then let the student regenerate the suffix in both conditions (Figure~\ref{fig:tstar-continuation-success}). The effect depends strongly on model scale: teacher-token and matched-control success are both $1.5\%$ at 1.7B, differ only slightly at 4B ($5.0\%$ vs.\ $4.4\%$), but separate substantially at 8B ($9.9\%$ vs.\ $1.7\%$). Thus, for Qwen3-8B, the gain cannot be explained by suffix resampling alone: changing only the branch token to the teacher-preferred action leads to substantially more successful student continuations. Of the 24 failed rollouts solved by teacher continuation at 8B, 17 are also solved by the intervened student continuation ($70.8\%$; Appendix Table~\ref{tab:continuation-overlap}).

\subsection{Implications for Trajectory-Level Learning}
\label{sec:analysis-bridge}

The diagnostic motivates disagreement-guided branching as a mechanism for collecting alternative student trajectories. Peak divergence identifies a candidate branch action even when the teacher-preferred token is not sufficient to repair the original response. Forcing this action yields substantially more successful student continuations than the matched resampling control in the larger model (Qwen3-8B). More generally, branching changes the successor prefixes visited by the student, creating contexts that are absent from the original rollout and can therefore receive new privileged-context teacher supervision.

\section{\method{}: Self-Distillation with Trajectory Intervention}
\label{sec:method}

Section~\ref{sec:analysis} shows that a teacher-preferred token can redirect student generation even when it does not repair the sampled response in place. Based on this observation, we propose \method{}, a simple branch--regenerate--distill algorithm that forces a teacher-selected token at a selected branch point, returns suffix generation to the student, and distills the resulting trajectory. We describe the algorithm in \S\ref{sec:algorithm}, motivate student-side regeneration in \S\ref{sec:method-student-generation}, and analyze the additional supervision created by regenerated trajectories in \S\ref{sec:method-collection-bottleneck}.

\begin{figure}[t]
  \centering
  \begin{subfigure}[t]{0.49\linewidth}
    \centering
    \includegraphics[width=\linewidth]{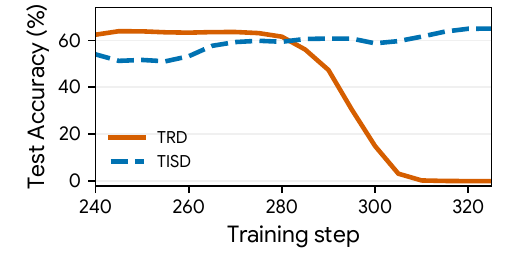}
    \caption{Avg@128 across training steps.}
    \label{fig:trd-biology-validation}
  \end{subfigure}
  \hfill
  \begin{subfigure}[t]{0.49\linewidth}
    \centering
    \includegraphics[width=\linewidth]{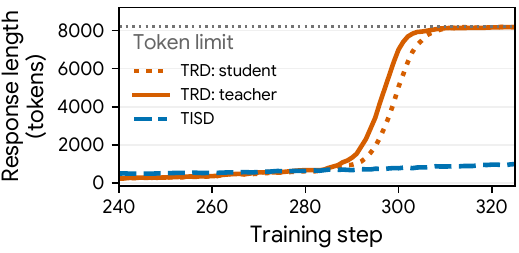}
    \caption{Mean response length across training steps.}
    \label{fig:trd-biology-response-length}
  \end{subfigure}
  \captionsetup{skip=5pt}
  \caption{Training collapse of distillation on fully teacher-regenerated rollouts (TRD) on SciKnowEval L3 biology (Qwen3-8B).}
  \label{fig:trd-biology-collapse}
  \vspace{-0.2cm}

\end{figure}

\subsection{\method{} Algorithm}
\label{sec:algorithm}

\paragraph{Trajectory construction.}
Consider one collection iteration for prompt $x$, with student and self-teacher parameters $\theta$ and $\phi$ fixed during trajectory collection. Let $\pi_{\theta}(\cdot\mid x,y_{<t})$ denote the student distribution over vocabulary $\mathcal{V}$, and let $\pi_{\phi}(\cdot\mid x,y_{<t},c)$ denote the self-teacher distribution conditioned additionally on privileged context $c$. The behavior policy $\beta_{\theta}$ is the sampling policy induced by $\pi_{\theta}$ under the rollout decoding configuration, including temperature and truncation. We distinguish $\beta_{\theta}$ from the underlying model distributions used for branch selection and distillation.

First draw an initial response $y=(y_1,\ldots,y_T)\sim\beta_{\theta}(\cdot\mid x)$ with length $T$, obtain verifier reward $R(y)\in\{0,1\}$, and construct privileged context $c_y$. A successful response is retained unchanged, while a failed response is regenerated when the required privileged context is available. For an eligible failure, \method{} selects the peak-divergence position
\begin{equation}
\tstar
=
\operatorname*{argmax}_{1\leq t<T}
\mathcal D_t\!\left(
\pi_{\theta}(\cdot\mid x,y_{<t}),
\pi_{\phi}(\cdot\mid x,y_{<t},c_y)
\right).
\label{eq:repo-branch-position}
\end{equation}
\method{} then selects the teacher's most likely action at this prefix,
\begin{equation}
\hat y_{\tstar}
=
\operatorname*{argmax}_{a\in\mathcal{V}}
\pi_{\phi}(a\mid x,y_{<\tstar},c_y).
\label{eq:repo-branch-action}
\end{equation}
The teacher-preferred action may coincide with the original student token, $\hat y_{\tstar}=y_{\tstar}$; in this case, resampling the suffix can still expose additional contexts for teacher supervision. \method{} preserves the original prefix, forces $\hat y_{\tstar}$, and returns generation to the student:
\begin{equation}
\tilde y
=
\left(y_{<\tstar}, \; \hat y_{\tstar}, \; \widetilde y_{>\tstar}\right),
\qquad
\tilde y_{>\tstar}
\sim
\beta_{\theta}(\cdot\mid x,y_{<\tstar},\hat y_{\tstar}).
\label{eq:repo-conditional-law}
\end{equation}
The retained trajectory is $z=\tilde y$ for an eligible failed response and $z=y$ otherwise. Thus, the training data include original trajectories as well as regenerated trajectories whose outcomes may remain unsuccessful or become successful.

\vspace{-0.1cm}\paragraph{Distillation on the retained trajectory.}
For both unchanged and regenerated responses, \method{} minimizes teacher--student divergence along the retained trajectory using the privileged context $c_y$ associated with the original rollout and the SDPO objective~\citep{hubotter2026sdpo}:
\begin{equation}
\ell_{\mathrm{dist}}(z,c_y;\theta)
=
\frac{1}{|z|}
\sum_{t=1}^{|z|}
\mathcal D_t\!\left(
\pi_\theta(\cdot\mid x,z_{<t}),
\sg\!\left[
\pi_{\phi}(\cdot\mid x,z_{<t},c_y)
\right]
\right),
\label{eq:repo-per-trace-loss}
\end{equation}
where $\mathcal D_t$ is the divergence metric and $\sg[\cdot]$ denotes stop-gradient.

\subsection{Why Return Generation to the Student?}
\label{sec:method-student-generation}

\method{} returns suffix generation to the student to limit the mismatch between training trajectories and student generation at inference. TRD~\citep{jiang2026trajectory} revises the response using a privileged-context teacher, so the resulting suffix may follow actions and prefixes that the unconditioned student would be unlikely to generate on its own. In \method{}, the teacher determines only the branch action, while all subsequent tokens are sampled from the student's conditional behavior policy (Equation~\ref{eq:repo-conditional-law}). This exposes an alternative branch while keeping the continuation generated by the student itself.

\vspace{-0.1cm}\paragraph{Training instability of teacher regeneration.}
Figure~\ref{fig:trd-biology-collapse} shows a TRD run on SciKnowEval L3 biology with Qwen3-8B. From step 245 to 320, validation Avg@128 falls from 64.8\% to 0.0\%. Over the same interval, the mean student- and teacher-generated response lengths increase from 282 and 270 tokens to 8,192 and 8,167 tokens, respectively, approaching the generation limit. A validation response from this run also contains repetitive claims of alignment with a ``reference solution'' that is unavailable at inference (Appendix~\ref{app:trd-validation-example}). These results demonstrate the training instability of full-teacher regeneration and motivate returning suffix generation to the student in \method{}.

\subsection{\method{} Creates New Opportunities for Teacher Supervision}
\label{sec:method-collection-bottleneck}

\begin{figure}[t]
  \centering
  \begin{subfigure}[t]{0.722\linewidth}
    \centering
    \includegraphics[width=\linewidth]{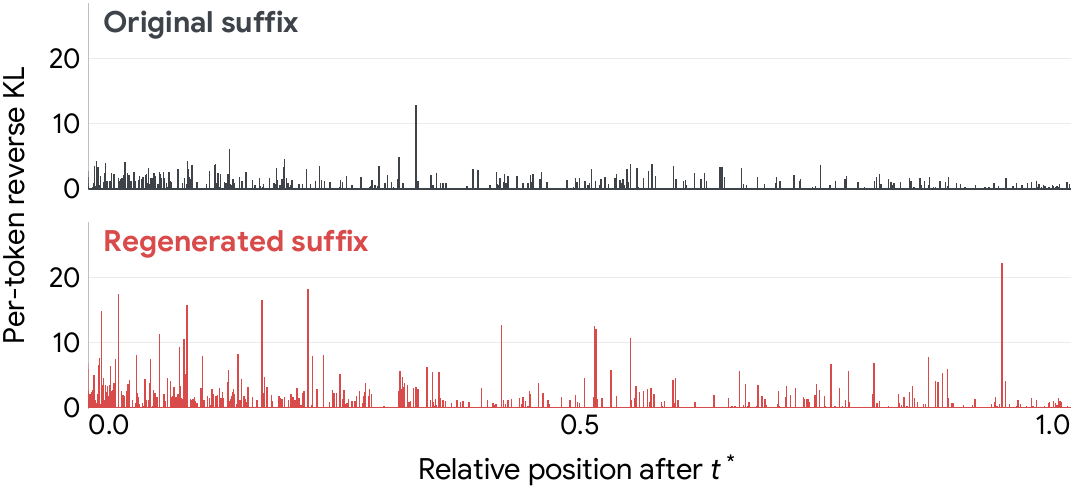}
    \caption{Aggregated token-level reverse KL profiles.}
    \label{fig:suffix-token-kl}
  \end{subfigure}
  \hfill
  \begin{subfigure}[t]{0.258\linewidth}
    \centering
    \includegraphics[width=\linewidth]{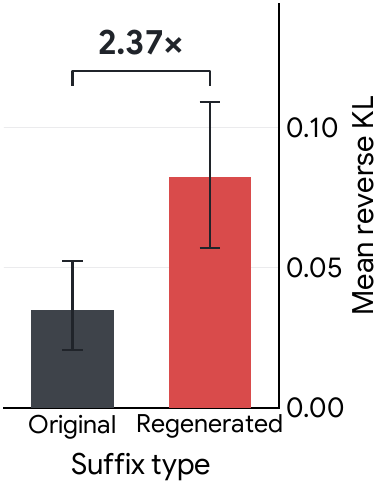}
    \caption{Mean reverse KL.}
    \label{fig:suffix-mean-kl}
  \end{subfigure}
  \caption{Token-level KL-divergence from original and token-changed regenerated suffixes. Results aggregate failed rollouts from LiveCodeBench v6 of the base model (Qwen3-8B).}
  \label{fig:suffix-kl-comparison}
  \vspace{-8pt}
\end{figure}

In standard OPSD, supervising a teacher-preferred action at an observed prefix does not provide supervision at the successor prefixes induced by that action unless the student enters the corresponding branch during collection. This limitation is especially relevant when the student assigns low probability to the teacher-preferred action and therefore rarely visits the branch under on-policy sampling. \method{} instead forces the teacher-selected action, lets the student regenerate the suffix from the intervened prefix, and distills the resulting trajectory using Equation~\ref{eq:repo-per-trace-loss}. This makes teacher supervision available at the branch-induced successor prefixes within the same collection iteration.

Importantly, this benefit is not limited to early training. Even when the teacher targets on the original trajectories are well fitted, the student's predictions at intervention-induced prefixes need not be similarly aligned with the teacher (Appendix~\ref{app:conditional-advantage}, Proposition~\ref{prop:ident-ambiguity}). Thus, regenerated trajectories can introduce teacher supervision constraints that are absent from the original trajectories.

\vspace{-0.1cm}\paragraph{Regenerated suffixes expose additional teacher--student discrepancies.}
In the Qwen3-8B LiveCodeBench v6 diagnostic, regenerated suffixes after intervention have a mean token reverse KL that is $2.37\times$ that of their paired original suffixes (Figure~\ref{fig:suffix-kl-comparison}). This indicates that a single teacher-selected action can move the student into successor contexts with substantially larger teacher--student distribution disagreement, even though the continuation itself is generated by the student. These newly visited contexts provide additional targets for teacher supervision, whose downstream training effect is evaluated in Section~\ref{sec:experiments}.

\section{Experiments}
\label{sec:experiments}

We evaluate \method{} in two OPSD settings: (1) coding with rich environment feedback and (2) science reasoning without rich environment feedback. We compare against GRPO~\citep{shao2024deepseekmath} and SDPO~\citep{hubotter2026sdpo} in both settings, and additionally against TRD~\citep{jiang2026trajectory} on LiveCodeBench. These methods differ in how training trajectories are constructed: SDPO distills along the original student-generated trajectory, TRD distills along a teacher-regenerated trajectory, and \method{} distills along a student-regenerated trajectory initiated by a teacher-selected branch.

\subsection{Learning with Rich Environment Feedback}

\paragraph{Setup.}
We train Qwen3-8B, Qwen3-14B, and Olmo3-7B-Instruct on LiveCodeBench v6 with rich environment feedback. For each failed rollout, the privileged context is a successful sibling when available, and otherwise execution feedback from the verifier. We compare \method{} with GRPO, SDPO, and TRD over 80 training updates. We select the checkpoint with the highest Avg@4 and report Pass@4 from the same checkpoint. Appendix~\ref{app:lcb-hparams} provides the full experimental details.

\begin{table}[t]
  \caption{Best Avg@4 and corresponding Pass@4 (\%) within 80-step training on LiveCodeBench v6 with rich environment feedback. Bold and underlining mark the highest and second-highest trained-method values in each column.}
  \label{tab:lcb-results}
  \centering
  \small
  \setlength{\tabcolsep}{0pt}
  
  \begin{tabular*}{\linewidth}{@{\hspace{3pt}}p{\dimexpr0.20\linewidth-6pt\relax}*{8}{>{\centering\arraybackslash}p{0.101\linewidth}}@{}}
    \toprule
    Method & \multicolumn{2}{c}{Qwen3-8B} & \multicolumn{2}{c}{Qwen3-14B} & \multicolumn{2}{c}{Olmo3-7B-Instruct} & \multicolumn{2}{c}{Average} \\
    \cmidrule(lr){2-3} \cmidrule(lr){4-5} \cmidrule(lr){6-7} \cmidrule(lr){8-9}
    & Avg & Pass & Avg & Pass & Avg & Pass & Avg & Pass \\
    \midrule

    Base & 27.5 & 36.1 & 28.6 & 35.1 & 32.1 & 42.7 & 29.4 & 38.0 \\
    + GRPO & 40.7 & 44.0 & 44.1 & 48.1 & 39.3 & \underline{48.1} & 41.4 & 46.7 \\
    + SDPO & \underline{48.1} & \underline{48.6} & \underline{48.9} & \underline{48.9} & \underline{50.8} & \textbf{53.4} & \underline{49.2} & \underline{50.3} \\
    + TRD & 39.3 & 41.2 & 48.1 & 48.1 & 38.0 & 40.5 & 41.8 & 43.3 \\
    + \textbf{\method{}} (ours) & \textbf{49.2} & \textbf{49.6} & \textbf{51.1} & \textbf{51.1} & \textbf{51.0} & \textbf{53.4} & \textbf{50.4} & \textbf{51.4} \\
    \bottomrule
  \end{tabular*}
  \vspace{-0.1cm}

\end{table}

\vspace{-0.1cm}\paragraph{Results.}
Table~\ref{tab:lcb-results} shows that \method{} achieves the highest Avg@4 across all three models and the highest Pass@4, tying SDPO on Olmo3-7B-Instruct. Averaged across models, \method{} reaches $50.4\%$ Avg@4 and $51.4\%$ Pass@4, compared with $49.2\%$ and $50.3\%$ for SDPO. These gains show that, even with rich execution feedback and dense token-level supervision, performance improves when the teacher-preferred action is used to collect an alternative student continuation. Unlike SDPO, which distills along the original student suffix, \method{} supervises the student-generated continuation that follows the teacher-selected branch. This result supports trajectory collection as an additional design dimension for using privileged feedback effectively.

\subsection{Learning without Rich Environment Feedback}

\paragraph{Setup.}
We study SciKnowEval L3, which contains undergraduate-level biology, chemistry, materials science, and physics problems. Following SDPO~\citep{hubotter2026sdpo}, \method{} and SDPO use a successful sibling as privileged context for a failed rollout when one is available. We additionally compare with off-policy and on-policy GRPO. We exclude TRD from this setting because full-teacher regeneration exhibits severe training instability on SciKnowEval (Section~\ref{sec:method-student-generation}). Training is limited to 10 hours. We report both step-matched and time-matched comparisons, using Avg@128 for checkpoint selection and reporting Pass@128 from the same checkpoint. Table~\ref{tab:science-reasoning} reports performance across the four domains, and Appendix~\ref{app:science-hparams} provides the full experimental details.

\vspace{-0.1cm}\paragraph{Results.}
Table~\ref{tab:science-reasoning} shows that \method{} achieves the highest average Avg@128 under both the equal-step and equal-time comparisons. Averaged across the four domains, \method{} reaches 73.5\% under the 200-step budget and 74.5\% under the 10-hour budget, improving over SDPO by 0.8 and 0.3 percentage points, respectively. Together with the LiveCodeBench results, these gains show that teacher-guided branching can improve OPSD both with rich execution feedback and when privileged supervision is limited to successful sibling trajectories.

\begin{table}[t]
  \caption{Best Avg@128 (\%) on SciKnowEval L3 without rich environment feedback (Qwen3-8B). Steps and Hours denote training budgets of 200 steps and 10 hours, respectively. Average is the unweighted mean across the four disciplines. Bold and underlining mark the highest and second-highest means in each column among the reported methods.}
  \label{tab:science-reasoning}
  \centering
  \small
  \setlength{\tabcolsep}{3pt}

  \begin{tabular*}{\linewidth}{@{\hspace{6pt}\extracolsep{\fill}}l*{10}{c}@{\hspace{6pt}}}
    \toprule
    Method & \multicolumn{2}{c}{Biology} & \multicolumn{2}{c}{Chemistry} & \multicolumn{2}{c}{Materials} & \multicolumn{2}{c}{Physics} & \multicolumn{2}{c}{Average} \\
    \cmidrule(lr){2-3} \cmidrule(lr){4-5} \cmidrule(lr){6-7} \cmidrule(lr){8-9} \cmidrule(lr){10-11}
    & Steps & Hours & Steps & Hours & Steps & Hours & Steps & Hours & Steps & Hours \\
    \midrule
    Base & 30.6 & 30.6 & 41.6 & 41.6 & 59.0 & 59.0 & 59.0 & 59.0 & 47.6 & 47.6 \\
    + GRPO & \underline{59.6} & \underline{63.8} & 73.9 & 75.2 & \textbf{78.5} & \textbf{80.1} & 73.6 & 73.6 & 71.4 & 73.2 \\
    + GRPO (on-policy) & 51.5 & 51.5 & 66.8 & 66.8 & 73.4 & 73.4 & 69.4 & 69.4 & 65.2 & 65.2 \\
    + SDPO & 57.4 & 61.1 & \underline{79.5} & \textbf{81.1} & \underline{77.5} & \underline{78.1} & \underline{76.3} & \underline{76.3} & \underline{72.7} & \underline{74.2} \\
    
    + \textbf{\method{}} (ours) & \textbf{60.7} & \textbf{63.9} & \textbf{80.4} & \underline{80.8} & 74.7 & 75.1 & \textbf{78.2} & \textbf{78.3} & \textbf{73.5} & \textbf{74.5} \\

    \bottomrule
  \end{tabular*}
\end{table}

\subsection{Additional Results}
\label{sec:additional-results}

\paragraph{Performance across training data scale.}
We additionally compare Qwen3-8B training on LCB-Large, which combines v1--v6, and LCB-Small, which contains v6, to examine performance and learning efficiency across training data scales. Figure~\ref{fig:lcb-learning-curve} reports Avg@4 through 80 updates. On LCB-Large, \method{} reaches $57.3$ Avg@4 and $87.8$ Pass@4, compared with the best baseline values of $54.3$ Avg@4 from GRPO and $86.3$ Pass@4 from SDPO. On LCB-Small, \method{} reaches $66.5$ Avg@4 and $92.4$ Pass@4, compared with $65.5$ and $92.1$ for SDPO. TRD reaches $59.9$ Avg@4 on LCB-Small, while its best Avg@4 on LCB-Large ($45.4$) remains below the base model ($49.6$). \method{} reaches the best baseline performance after $33\%$ fewer updates on LCB-Large and $56\%$ fewer updates on LCB-Small.

\begin{figure}[t]
  \centering
  \includegraphics[width=\linewidth]{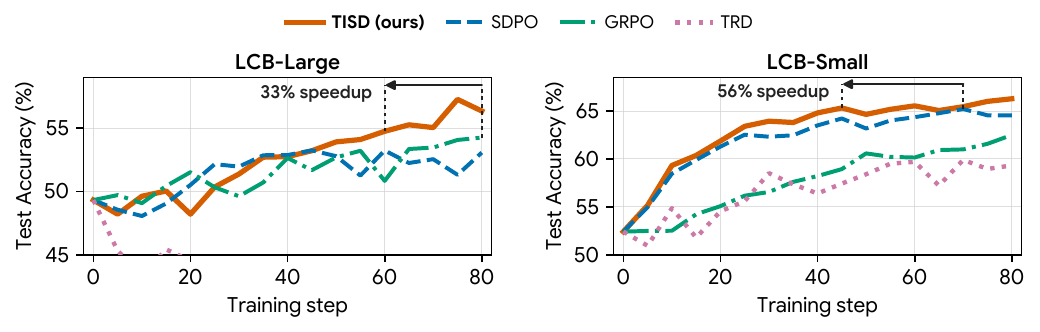}
  \vspace{-0.65cm}
  \caption{Mean test-case pass rate (Avg@4, \%) across training data scales, through 80 training steps in the rich-feedback setting (Qwen3-8B). LCB-Large combines v1--v6, and LCB-Small contains v6. \method{} first exceeds GRPO's maximum on LCB-Large with a $33\%$ speedup and reaches SDPO's maximum on LCB-Small with a $56\%$ speedup.}
  \label{fig:lcb-learning-curve}
  \vspace{-0.2cm}
\end{figure}

\vspace{-0.1cm}\paragraph{Does peak divergence identify a useful branch?}
The gains from regeneration depend on where the branch begins. We replace \method{}'s peak-divergence selector with either a uniformly random response position or the first student--teacher argmax disagreement, while keeping the remaining training and evaluation configuration fixed. On LCB-Small, peak-divergence branching reaches $66.5$ Avg@4, compared with $65.2$ for random-position branching and $64.6$ for first-disagreement branching (Figure~\ref{fig:lcb-branch-position-ablation}). Relative to SDPO at $65.5$ Avg@4, only peak-divergence branching improves performance among the evaluated selectors, supporting divergence magnitude as a useful criterion for selecting the branch position. Retaining the student's original token at the peak-divergence position and resampling the suffix (\emph{Student own token}) reaches only $64.8$ Avg@4, indicating that both the branch position and the teacher-selected action contribute to the gain.

A complementary diagnostic shows that the preferred intervention position depends on who generates the suffix. On failed Qwen3-8B LiveCodeBench v6 rollouts at initialization, forcing the teacher token and returning generation to the student succeeds more often at peak divergence than at the first argmax disagreement ($9.9\%$ vs.\ $6.1\%$). When the teacher instead generates the suffix, the first argmax disagreement yields higher success ($17.1\%$ vs.\ $13.3\%$; Appendix Table~\ref{tab:steer_qwen8b_positions}). Since \method{} returns suffix generation to the student, these results further support peak divergence as the branch selector for \method{}.

\vspace{-0.1cm}\paragraph{Is one regeneration sufficient?}
The standard procedure performs a single branch--regenerate operation. We additionally allow the regenerated trajectory to branch again at a strictly later position, while keeping the original privileged context fixed and making no additional verifier calls between rounds. Under the same 80-update schedule, two total rounds reach $65.2$ Avg@4 and $90.8$ Pass@4, while three rounds reach $65.5$ Avg@4 and $90.1$ Pass@4. In comparison, single-round regeneration achieves $66.5$ Avg@4 and $92.4$ Pass@4, the highest performance among the evaluated depths (Figure~\ref{fig:lcb-recursive-regeneration}). Thus, additional regeneration rounds do not improve performance under this fixed-feedback training setup, supporting the single-round design of \method{}.

\begin{figure}[t]
  \centering
  \begin{minipage}[t]{0.49\linewidth}
    \centering
    \includegraphics[width=\linewidth]{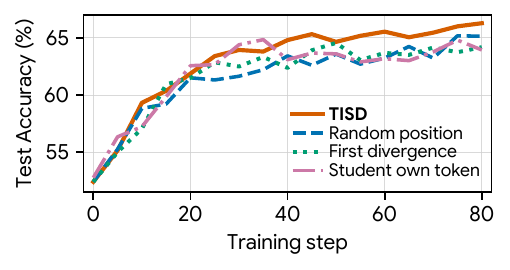}
  \vspace{-0.65cm}
\captionof{figure}{Branch-position and token ablations on Qwen3-8B LiveCodeBench v6 Avg@4 through update 80. Position controls use random positions or first argmax disagreement; Student own token retains the  token at peak reverse KL and resamples the suffix. }
    \label{fig:lcb-branch-position-ablation}
    
  \end{minipage}\hfill
  \begin{minipage}[t]{0.49\linewidth}
    \centering
    \includegraphics[width=\linewidth]{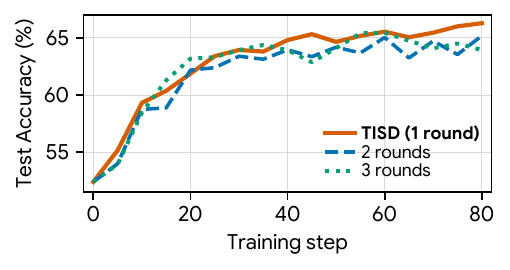}
  \vspace{-0.65cm}
\captionof{figure}{Effect of regeneration depth on Qwen3-8B LiveCodeBench v6 Avg@4 through update 80. Standard \method{} uses one branch--regenerate round; recursive variants apply two or three total rounds at strictly increasing response positions. }
    \label{fig:lcb-recursive-regeneration}
  \end{minipage}
  \vspace{-0.2cm}
\end{figure}

\section{Discussion and Conclusion}
\label{sec:discussion}

On-policy self-distillation depends not only on the teacher targets but also on the trajectories on which they are evaluated. Our results identify trajectory collection as an important design dimension in privileged-context self-distillation: teacher guidance can determine which student contexts receive supervision, in addition to what the student learns at an observed prefix. Based on this insight, we introduce \method{}, which uses a teacher-selected action to initiate an alternative student continuation and distills the resulting trajectory. Improvements across coding and science support teacher-guided branching as a way to expose additional contexts for privileged supervision.

Our analysis further shows that teacher--student disagreement can serve as a trajectory-steering signal rather than only a local correction signal. A teacher-preferred token can redirect subsequent student generation even when substituting that token alone does not repair the response. This distinction suggests that trajectory interventions should be evaluated not only by immediate success, but also by the new contexts they expose for teacher supervision.

\vspace{-0.1cm}\paragraph{Limitations.}
The effectiveness of \method{} depends on the student's ability to continue from a teacher-selected action and on the quality of the privileged teacher targets. Also, branch selection and suffix regeneration introduce additional training-time computation. Reducing this overhead through more efficient branch selection or generation is an important direction for future work.

\section*{AI Use Statement}
We used generative AI tools to (1) polish writing, (2) LLM-judge analysis in Section~\ref{sec:analysis}, and (3) implement methods. We have not used generative AI tools otherwise. We have reviewed all AI-assisted work. We take responsibility for the final content of this work, including text, claims or artifacts produced with the aid of generative AI.

\section*{Reproducibility Statement}
To support reproducibility, Section~\ref{sec:algorithm} specifies the trajectory construction and distillation objective of \method{}, and Appendix~\ref{app:implementation} provides pseudocode. Section~\ref{sec:experiments} describes the datasets, baselines, evaluation metrics, and checkpoint-selection criteria, with training hyperparameters detailed in Appendix~\ref{app:hyperparameters}. Appendix~\ref{app:eval} documents the semantic annotation and continuation-evaluation protocols used in the diagnostic analyses. We will publicly release the source code upon acceptance.

\newpage
\appendix
\section{Implementation of \method{}}
\label{app:implementation}

Figure~\ref{fig:tisd-pseudocode} summarizes the branch--regenerate--distill procedure of \method{}, using the self-distillation objective of SDPO~\citep{hubotter2026sdpo}.

\begin{figure}[ht]
  \centering
  \begin{tcolorbox}[colback=CodeBackground,boxrule=0pt,arc=2mm,left=22pt,right=8pt,top=7pt,bottom=7pt]
\begin{lstlisting}[style=tisd-pseudocode]
def tisd_step(prompts, student, teacher):
    # 1. Sample responses, verify, and build privileged contexts.
    batch = collect_rollouts(student, prompts)
    retained = []

    # 2. Retain original trajectories; regenerate eligible failures.
    for x, y, reward, c in batch:
        z = y
        if reward == 0 and eligible(y, c):
            t = peak_divergence_pos(student, teacher, x, y, c)
            token = teacher.argmax(x, y[:t], c)
            prefix = y[:t] + [token]
            z = prefix + student.generate(x, prefix)
        retained.append((x, z, c))

    # 3. Distill the retained trajectories.
    loss = distillation_loss(student, teacher, retained)
    student.update(loss)
    teacher.update_ema(student)
\end{lstlisting}
  \end{tcolorbox}
  \caption{\method{} pseudocode. Each record contains a prompt \texttt{x}, response \texttt{y}, binary reward, and privileged context \texttt{c}. \texttt{peak\_divergence\_pos} selects the eligible position with the largest teacher--student divergence. Loss and optimization settings appear in Appendix~\ref{app:hyperparameters}.}
  \label{fig:tisd-pseudocode}
\end{figure}

\section{Preliminary Analysis: Full Details}
\label{app:analysis}

This appendix provides the full diagnostic analysis supporting \method{}'s branch selection and regeneration design (Section~\ref{sec:analysis}). The source trajectories come from SDPO training of Qwen3-$\{$1.7B, 4B, 8B$\}$~\citep{yang2025qwen3} on LiveCodeBench v6~\citep{jain2024livecodebench}; a typical training step collects $256$ rollouts ($32$ problems $\times$ $8$ rollouts). We focus on failed rollouts because \method{} branches from them and because they allow us to examine the teacher guidance available beyond the original student trajectory. We first describe the analysis protocol, then provide additional breakdowns of the diagnostics in Section~\ref{sec:analysis} and extend the intervention analysis to other architectures and science domains.

\subsection{Analysis Protocol}
\label{app:eval}

The coding diagnostics combine semantic annotation with execution-based evaluation. For semantic annotation, GPT-5.4 (temperature $0$, $1024$ completion tokens) receives the problem statement, the student's full solution, and a $15$-token context window centered on the selected position, and classifies both the student and teacher tokens by \emph{semantic role} (algorithmic, naming, formatting, non-code text).

For continuation evaluation, we preserve the student prefix before the selected position, force the teacher's argmax token, and greedily generate the suffix using either the teacher (\emph{teacher continuation}) or the student (\emph{student continuation}). The matched student control instead forces the student's original token before regenerating the suffix. On LiveCodeBench, the resulting Python block is extracted and executed against the full private test suite; a response is correct if and only if it passes all test cases.

On SciKnowEval L3, correctness is binary: the extracted answer option must exactly match the reference answer. We apply the same criterion to the original rollout and each regenerated continuation. The cross-domain analysis uses Qwen3-8B rollouts with $32$ problems and eight rollouts per problem in each science domain. Success rates are conditioned on original failures, and pooled science rates aggregate outcomes over all $553$ failed rollouts.

\subsection{KL Concentration CDF}
\label{app:kl-conc}

The KL concentration ratio $C = \mathrm{KL}(\tstar)/\sum_t \mathrm{KL}(t)$ measures the fraction of total response KL carried by the peak-divergence position. Figure~\ref{fig:concentration-cdf} shows its cumulative distribution across $202$ failed Qwen3-1.7B rollouts. The median concentration is ${\sim}19.5\%$, with $\tstar$ accounting for more than $50\%$ of total response KL in some cases. Across model scales, the concentration decreases from $19.5\%$ at 1.7B to ${\approx}9\%$ at 4B and 8B. Even at ${\approx}9\%$, the peak position carries roughly $40\times$ the per-position average KL.

\begin{figure}[ht]
  \centering
  \includegraphics[width=0.6\linewidth]{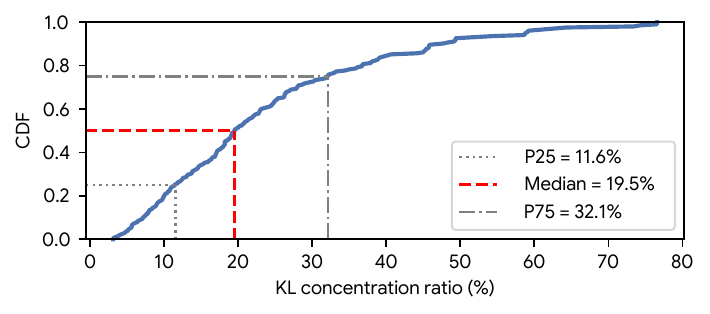}
  \caption{CDF of the reverse-KL concentration ratio $\mathrm{KL}(\tstar)/\sum_t \mathrm{KL}(t)$ on $202$ failed Qwen3-1.7B LiveCodeBench v6 rollouts.}
  \label{fig:concentration-cdf}
\end{figure}

\subsection{Semantic Role Breakdown}
\label{app:semantic}

Table~\ref{tab:semantic-role} reports GPT-5.4 semantic annotations of the student and teacher tokens at $\tstar$ on failed rollouts. At larger model scales, both student and teacher tokens at peak divergence are predominantly classified as non-code text, while algorithmic roles become less frequent.

\begin{table}[ht]
\caption{GPT-5.4 semantic annotations of student and teacher tokens at $\tstar$ on failed Qwen3-\{1.7B, 4B, 8B\} LiveCodeBench v6 rollouts. Entries are percentages within each model scale.}
  \label{tab:semantic-role}
  \centering
  \small
  \begin{tabular*}{0.7\linewidth}{@{\hspace{6pt}\extracolsep{\fill}}llccc@{\hspace{6pt}}}
    \toprule
    & Role & 1.7B & 4B & 8B \\
    \midrule
    \multirow{4}{*}{Student}
    & Algorithmic   & 36.1\% & 11.6\% & 11.6\% \\
    & Naming        & 11.4\% &  3.9\% &  5.5\% \\
    & Formatting    &  7.4\% & 13.3\% & 12.7\% \\
    & Non-code text & 45.0\% & 71.3\% & 70.2\% \\
    \midrule
    \multirow{4}{*}{Teacher}
    & Algorithmic   & 38.6\% & 10.5\% & 14.4\% \\
    & Naming        &  9.9\% &  4.4\% &  3.9\% \\
    & Formatting    &  8.4\% & 11.6\% & 13.8\% \\
    & Non-code text & 43.1\% & 73.5\% & 68.0\% \\
    \bottomrule
  \end{tabular*}
\end{table}

Figure~\ref{fig:transition-matrix} shows the student-to-teacher semantic-role transition matrix. The diagonal dominates across model scales, indicating that the student and teacher tokens at $\tstar$ often belong to the same broad semantic category. Thus, large teacher--student divergence frequently reflects differences within a semantic role, such as alternative variable names or different comments, rather than a change in the type of token being produced.

\begin{figure}[ht]
  \centering
  \includegraphics[width=\linewidth]{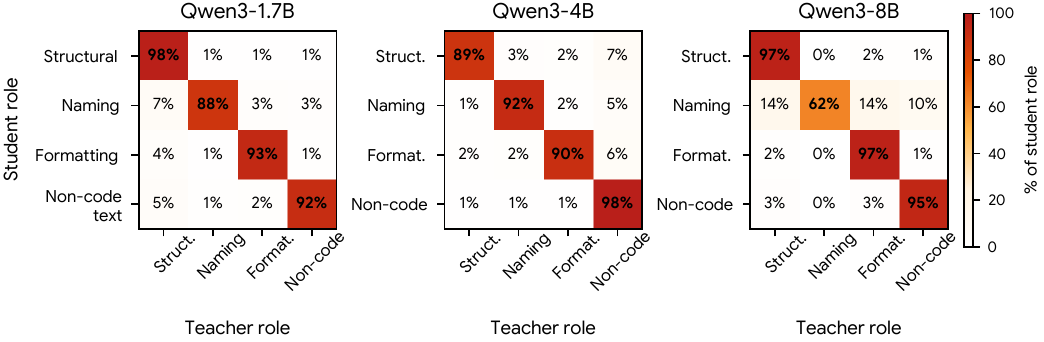}
  \caption{Student-to-teacher semantic-role transitions at $\tstar$ on failed Qwen3-\{1.7B, 4B, 8B\} LiveCodeBench v6 rollouts. Each row shows the distribution of teacher-token roles conditioned on the corresponding student-token role. The dominant diagonal shows that disagreement usually remains within a semantic category.}
  \label{fig:transition-matrix}
\end{figure}

\subsection{Effect of Branch-Selection Position}
\label{app:branch-position}

Peak-divergence selection yields the highest student continuation success among the evaluated branch selectors on failed Qwen3-8B LiveCodeBench rollouts (Table~\ref{tab:steer_qwen8b_positions}). Forcing the teacher token at peak divergence gives $9.9\%$ success, compared with $6.1\%$ at first divergence and $1.7\%$ at random positions. The corresponding original-token controls yield $1.7\%$, $2.2\%$, and $3.3\%$, respectively. Teacher continuation shows a different ordering: first divergence reaches $17.1\%$, exceeding $13.3\%$ at peak divergence. Thus, the branch position that best supports student regeneration differs from that which best supports teacher continuation. Since \method{} returns suffix generation to the student, these results support peak divergence as its branch-selection criterion.

\begin{table}[ht]
  \caption{Effect of branch-position selection on $181$ failed Qwen3-8B LiveCodeBench v6 rollouts. Semantic roles refer to the student token at the selected position and are assigned by GPT-5.4. Execution success is measured after either teacher continuation or student regeneration from the forced teacher or original token. \emph{First divergence} denotes the first position where the student and teacher argmax tokens differ.}
  \label{tab:steer_qwen8b_positions}
  \centering
  \small
  \begin{tabular*}{0.82\linewidth}{@{\extracolsep{\fill}}llccc@{\hspace{6pt}}}
    \toprule
    & & Peak divergence & First divergence & Random \\
    \midrule
    \multicolumn{2}{l}{\textit{Semantic role ({student token at $\tstar$})}} \\
    & Algorithmic   & 11.6\% &  0.6\% & 20.4\% \\
    & Naming        &  5.5\% &  0.0\% &  3.9\% \\
    & Formatting    & 12.7\% & 15.5\% & 11.0\% \\
    & Non-code text & 70.2\% & 84.0\% & 64.6\% \\
    \midrule
    \multicolumn{2}{l}{\textit{Execution-based success rate at $\tstar$}} \\
    & Teacher continuation
      & 13.3\% & 17.1\% & 8.3\% \\
    & Student cont.\ (teacher token)
      & 9.9\% & 6.1\% & 1.7\% \\
    & Student cont.\ (original token)
      & 1.7\% & 2.2\% & 3.3\% \\
    \bottomrule
  \end{tabular*}
\end{table}

\subsection{Persistence During Training}
\label{app:stepwise}

Non-code text remains the largest semantic category at peak divergence throughout Qwen3-1.7B SDPO training. Table~\ref{tab:stepwise} reports checkpoints from steps $0$ to $80$. Over this interval, the overall rollout pass rate increases from $0.211$ to $0.434$, while the algorithmic share of peak-divergence positions among failed rollouts decreases from $36.1\%$ to $26.2\%$. Semantic-role percentages are computed separately over the failed rollouts at each checkpoint.

\begin{table}[ht]
\caption{Evolution of rollout pass rate and semantic roles at the peak-divergence position during Qwen3-1.7B SDPO training on LiveCodeBench v6, using $256$ rollouts per update. Pass rate is computed over all rollouts, while GPT-5.4 semantic-role percentages are computed over failed rollouts at each update; the corresponding failure counts are shown in parentheses.}  \label{tab:stepwise}
  \centering
  \small
  \begin{tabular*}{\linewidth}{@{\extracolsep{\fill}}llccccc@{\hspace{6pt}}}
    \toprule
    & & \makecell{Step 0\\(202)} & \makecell{Step 20\\(183)} & \makecell{Step 40\\(192)} & \makecell{Step 60\\(176)} & \makecell{Step 80\\(145)} \\
    \midrule
    \multicolumn{2}{l}{Pass rate} & 0.211 & 0.285 & 0.250 & 0.312 & 0.434 \\
    \midrule
    \multicolumn{2}{l}{\textit{Semantic role ({student token at $\tstar$})}} \\
    & Algorithmic   & 36.1\% & 35.5\% & 36.5\% & 35.8\% & 26.2\% \\
    & Naming        & 11.4\% & 15.3\% & 12.0\% & 14.8\% &  9.7\% \\
    & Formatting    &  7.4\% &  8.2\% & 13.5\% & 11.4\% & 15.9\% \\
    & Non-code text & 45.0\% & 41.0\% & 38.0\% & 38.1\% & 48.3\% \\
    \bottomrule
  \end{tabular*}
\end{table}

\subsection{Contribution of Failed-Rollout Supervision}
\label{app:failed-supervision}

We ablate supervision on failed rollouts during Qwen3-1.7B SDPO training on LiveCodeBench. At step $80$, applying the KL loss only to successful rollouts yields $0.372$ validation accuracy, compared with $0.480$ when supervising both successful and failed rollouts, a difference of $10.8$ percentage points. This fixed-step comparison motivates analyzing the teacher guidance available on failed rollouts, which remain useful training examples under the standard SDPO objective.

\subsection{Cross-Architecture Validation}
\label{app:crossarch}

We replicate the Section~\ref{sec:analysis} diagnostic on OLMo3-7B-Instruct~\citep{olmo2025olmo3} at step $0$, using $221$ non-truncated LiveCodeBench v6 rollouts after excluding $35$ responses that reached the $16{,}384$-token limit. The main patterns observed with Qwen3 persist (Table~\ref{tab:olmo3}). On failed rollouts, non-code text accounts for $81.4\%$ of peak-divergence positions, and teacher continuation succeeds on $24.8\%$ of failures. The trajectory-steering effect also reproduces: forcing the teacher token and returning generation to the student yields $10.1\%$ success, compared with $6.2\%$ when retaining the original student token. Thus, the qualitative findings on peak-divergence semantics and teacher-guided student regeneration extend beyond the Qwen3 architecture.

\begin{table}[ht]
\caption{Reverse-KL diagnostic on $221$ non-truncated OLMo3-7B-Instruct LiveCodeBench v6 rollouts. Entries report semantic-role distributions at $\tstar$ and execution success after continuation from $\tstar$, separately for failed and successful original rollouts.}  \label{tab:olmo3}
  \centering
  \small
  \begin{tabular*}{0.6\linewidth}{@{\extracolsep{\fill}}llcc@{\hspace{3pt}}}
    \toprule
    & & \makecell{Fail\\(129)} & \makecell{Succ\\(92)} \\
    \midrule
    \multicolumn{2}{l}{\textit{Semantic role (student token at $\tstar$)}} \\
    & Algorithmic   &  7.8\% & 15.2\% \\
    & Naming        &  3.1\% & 10.9\% \\
    & Formatting    &  7.8\% & 10.9\% \\
    & Non-code text & 81.4\% & 63.0\% \\
    \midrule
    \multicolumn{2}{l}{\textit{Execution-based success rate at $\tstar$}} \\
    & Teacher continuation & 24.8\% & 69.6\% \\
    & Student cont. (teacher token) & 10.1\% & -- \\
    & Student cont. (original token) & 6.2\% & -- \\
    \bottomrule
  \end{tabular*}
\end{table}

\subsection{Cross-Domain Validation}
\label{app:crossdomain}

Teacher-token intervention also improves student continuation success on SciKnowEval L3 (Table~\ref{tab:tstar_fail_lcb_science_qwen8b}). Across $553$ failed Qwen3-8B science rollouts, student continuation from the teacher token succeeds on $39.4\%$, compared with $24.4\%$ for the original-token control, with the improvement appearing in all four domains. The corresponding LiveCodeBench rates are $9.9\%$ and $1.7\%$. Teacher continuation itself is substantially more successful on science, reaching $93.5\%$ compared with $13.3\%$ on coding failures. These results show that teacher-selected branch actions improve student continuation across both task settings, while the teacher's ability to complete the trajectory differs substantially between science and coding.

\begin{table*}[ht]
\caption{Verifier-based intervention on failed Qwen3-8B SciKnowEval L3 rollouts at SDPO update 0. The teacher is conditioned on the first successful same-problem rollout in source order, excluding the target rollout, and $\tstar$ is the valid response position with maximum reverse KL. Each domain contains $32$ problems with eight rollouts per problem. Numbers in parentheses denote failed-rollout counts.}  \label{tab:tstar_fail_lcb_science_qwen8b}
  \centering
  \small
  \setlength{\tabcolsep}{3.5pt}
  \begin{tabular*}{\textwidth}
    {@{\extracolsep{\fill}}llccccc@{\hspace{6pt}}}
    \toprule
    & & \makecell{Biology\\(148)}
    & \makecell{Chemistry\\(143)}
    & \makecell{Physics\\(128)}
    & \makecell{Materials\\(134)}
    & \makecell{Pooled\\(553)} \\
    \midrule
    \multicolumn{2}{l}{\textit{Execution-based success rate at $\tstar$}} \\
    & \hspace{1em}Teacher continuation
      & 96.6\% & 93.0\% & 93.8\% & 90.3\% & 93.5\% \\
    & \hspace{1em}Student cont.\ (teacher token)
      & \textbf{31.1\%} & \textbf{41.3\%} & \textbf{40.6\%} & \textbf{45.5\%} & \textbf{39.4\%} \\
    & \hspace{1em}Student cont.\ (original token)
      & 22.3\% & 25.9\% & 30.5\% & 19.4\% & 24.4\% \\
    \bottomrule
  \end{tabular*}
\end{table*}

\subsection{Coverage of GRPO Advantage}
\label{app:grpo}

For binary rewards, group-relative outcome supervision is unavailable for many failed rollouts. With $K=8$ rollouts per problem at update 0 (Table~\ref{tab:grpo-coverage}), $72\%$ of problem groups are all-fail for Qwen3-1.7B and $59\%$ for Qwen3-4B/8B, accounting for $91\%$ and $84\%$ of all failed rollouts, respectively. In an all-fail or all-success group, GRPO assigns zero normalized advantage to every rollout. Privileged-context self-distillation can instead provide token-level supervision on each failed rollout with valid context whenever the teacher and student disagree. \method{} further uses this disagreement to select a branch and supervise the resulting student continuation, including for failed rollouts with zero group-relative advantage.

\begin{table}[ht]
  \centering
\caption{Coverage of group-relative outcome supervision for Qwen3-\{1.7B, 4B, 8B\} on $32$ LiveCodeBench v6 problems at update 0, with eight rollouts per problem. Counts and percentages partition problem groups by verifier outcome and failed rollouts by whether their normalized GRPO advantage is zero. All-fail and all-success groups receive zero normalized advantage.}  \label{tab:grpo-coverage}
  \small
  \begin{tabular}{lccc}
    \toprule
    & {1.7B} & {4B} & {8B} \\
    \midrule
    \multicolumn{4}{l}{\textit{Problem groups (out of 32)}} \\
    \quad All-fail    & 23 (71.9\%) & 19 (59.4\%) & 19 (59.4\%) \\
    \quad All-success &  5 (15.6\%) &  6 (18.8\%) &  5 (15.6\%) \\
    \quad Mixed       &  4 (12.5\%) &  7 (21.9\%) &  8 (25.0\%) \\
    \midrule
    \multicolumn{4}{l}{\textit{Failed rollouts}} \\
    \quad Total              & 202 & 181 & 181 \\
    \quad GRPO advantage $=0$     & 184 (91.1\%) & 152 (84.0\%) & 152 (84.0\%) \\
    \quad GRPO advantage $\neq 0$ & 18 (8.9\%)  & 29 (16.0\%) & 29 (16.0\%) \\
    \bottomrule
  \end{tabular}
\end{table}

\subsection{Continuation Outcomes by Semantic Role}
\label{app:cross-tab}

Successful teacher continuations can begin at non-code peak-divergence positions. Table~\ref{tab:cross-tab} stratifies teacher continuation success by the semantic role at $\tstar$. At 4B and 8B, non-code positions have higher observed success ($13.2\%$ and $15.7\%$) than algorithmic positions ($4.8\%$ and $0.0\%$), while the ordering reverses at 1.7B ($3.3\%$ vs.\ $12.3\%$). Because these comparisons are based on naturally occurring peak positions with small role-specific sample sizes, we do not interpret semantic role alone as a reliable indicator of branch quality.

\begin{table}[ht]
\caption{Verifier success of teacher continuations from $\tstar$, stratified by GPT-5.4 semantic role, on failed Qwen3-\{1.7B, 4B, 8B\} LiveCodeBench v6 rollouts. Each entry reports successful continuations over failed rollouts assigned to the corresponding role.}  \label{tab:cross-tab}
  \centering
  \small
  \begin{tabular}{llccc}
    \toprule
    & & 1.7B & 4B & 8B \\
    \midrule
    \multicolumn{5}{l}{\textit{By semantic role}
\hfill \scriptsize(correct / total)} \\
    & Algorithmic  & 9/73 (12.3\%) & 1/21 (4.8\%)   & 0/21 (0.0\%) \\
    & Naming  & 4/23 (17.4\%) & 0/7 (0.0\%)    & 2/10 (20.0\%) \\
    & Formatting  & 0/15 (0.0\%)  & 1/24 (4.2\%)   & 2/23 (8.7\%) \\
    & Non-code text  & 3/91 (3.3\%)  & 17/129 (13.2\%) & 20/127 (15.7\%) \\
    \bottomrule
  \end{tabular}
\end{table}

\paragraph{Overlap of teacher and student successes.}
On the same $181$ failed Qwen3-8B rollouts, teacher and teacher-token student continuations succeed together in $17$ cases (Table~\ref{tab:continuation-overlap}). The teacher succeeds on $24$ rollouts in total, so the intervened student reproduces $17/24=70.8\%$ of these successful continuations. Outcomes are paired on the same original rollouts, and all counts are reported at the rollout level.

\begin{table}[ht]
  \centering
\caption{Paired continuation outcomes on $181$ failed Qwen3-8B LiveCodeBench v6 rollouts. Both conditions preserve the original prefix, force the teacher's argmax token at peak divergence, and differ only in whether the teacher or student generates the suffix. Success requires passing all private tests.}  \label{tab:continuation-overlap}
  \small
  \begin{tabular}{lcc}
    \toprule
    & Student succeeds & Student fails \\
    \midrule
    Teacher succeeds & 17 & 7 \\
    Teacher fails & 1 & 156 \\
    \bottomrule
  \end{tabular}
\end{table}

\subsection{Qualitative Example of Trajectory Redirection}
\label{app:trajectory-redirection-example}

Figure~\ref{fig:inversion-example} illustrates trajectory redirection in a failed Qwen3-8B LiveCodeBench response. Intervening on a naming token changes the student's subsequent reasoning, leading it to a different algorithm that avoids the original bug.

\begin{figure}[ht]
\centering
\begin{minipage}[t]{0.48\linewidth}
\centering\small\textbf{Original student trace}\vspace{2pt}
\begin{tcolorbox}[colback=gray!8, colframe=gray!40, colbacktitle=gray!20, coltitle=black, boxrule=0.5pt, arc=1.5pt,
  left=5pt, right=5pt, top=4pt, bottom=4pt, title={\scriptsize\sffamily Pre-code reasoning trace}]
\ttfamily\scriptsize
We need to count triples (i, j, k) where
S[i]='A', S[j]='B', S[k]='C' and
j - i = k - j (equally spaced).\par\smallskip
This means \colorbox{red!20}{j} = (i + k) // 2,
so we can iterate over i and k.
\end{tcolorbox}
\vspace{-4pt}
\begin{tcolorbox}[colback=gray!8, colframe=gray!40, colbacktitle=gray!20, coltitle=black, boxrule=0.5pt, arc=1.5pt,
  left=5pt, right=5pt, top=4pt, bottom=4pt, title={\scriptsize\sffamily Generated code}]
\ttfamily\scriptsize
for i in range(n):\par
\hspace*{1em}for k in range(i + 2, n):\par
\hspace*{2em}j = (i + k) // 2\par
\hspace*{2em}if S[i]=='A' and S[j]=='B'\par
\hspace*{3em}and S[k]=='C':\par
\hspace*{3em}count += 1\par
$\cdots$
\end{tcolorbox}
\vspace{-4pt}
\begin{tcolorbox}[colback=red!5, colframe=red!40, boxrule=0.5pt, arc=1.5pt,
  left=5pt, right=5pt, top=3pt, bottom=3pt]
\scriptsize\textcolor{red!70!black}{$\boldsymbol{\times}$~\textbf{Incorrect:} no parity check for \texttt{(i+k)\%2==0}; counts invalid half-integer indices}
\end{tcolorbox}
\end{minipage}
\hfill
\begin{minipage}[t]{0.48\linewidth}
\centering\small\textbf{Student continued with forced teacher token}\vspace{2pt}
\begin{tcolorbox}[colback=gray!8, colframe=gray!40, colbacktitle=gray!20, coltitle=black, boxrule=0.5pt, arc=1.5pt,
  left=5pt, right=5pt, top=4pt, bottom=4pt, title={\scriptsize\sffamily Pre-code reasoning trace}]
\ttfamily\scriptsize
We need to count triples (i, j, k) where
S[i]='A', S[j]='B', S[k]='C' and
j - i = k - j (equally spaced).\par\smallskip
This means \colorbox{green!20}{k} = 2j - i,
so we can iterate over i and j.
\end{tcolorbox}
\vspace{-4pt}
\begin{tcolorbox}[colback=gray!8, colframe=gray!40, colbacktitle=gray!20, coltitle=black, boxrule=0.5pt, arc=1.5pt,
  left=5pt, right=5pt, top=4pt, bottom=4pt, title={\scriptsize\sffamily Generated code}]
\ttfamily\scriptsize
for i in range(n):\par
\hspace*{1em}for j in range(i + 1, n):\par
\hspace*{2em}k = 2 * j - i\par
\hspace*{2em}if k < n and S[i]=='A'\par
\hspace*{3em}and S[j]=='B' and S[k]=='C':\par
\hspace*{3em}count += 1\par
$\cdots$
\end{tcolorbox}
\vspace{-4pt}
\begin{tcolorbox}[colback=green!5, colframe=green!40!black, boxrule=0.5pt, arc=1.5pt,
  left=5pt, right=5pt, top=3pt, bottom=3pt]
\scriptsize\textcolor{green!50!black}{$\boldsymbol{\checkmark}$~\textbf{Correct:} iterates (i,j), computes k=2j$-$i; k is always an integer}
\end{tcolorbox}
\end{minipage}
\caption{Example of trajectory redirection on a failed Qwen3-8B LiveCodeBench response. The peak-divergence token is the naming token \colorbox{red!20}{\small \texttt{j}}, which carries $12.3\%$ of response KL. Forcing the teacher-preferred \colorbox{green!20}{\small \texttt{k}} and regenerating with the student changes the algorithm from iterating over $(i,k)$ to iterating over $(i,j)$, producing an integer $k=2j-i$ and a correct solution.}
\label{fig:inversion-example}
\end{figure}

\subsection{Repetitive Meta-Commentary in a TRD Validation Response}
\label{app:trd-validation-example}

A TRD student validation response can exhibit solution-rewriting language without providing task-specific reasoning. Figure~\ref{fig:trd-validation-example} shows an example accompanying the discussion in Section~\ref{sec:method-student-generation}. The prompt asks for a protein folding stability score, but the response repeatedly asserts fidelity to a ``reference solution'' that is absent from the displayed prompt. It selects option D without explaining how the sequence supports that choice.

\begin{figure}[ht]
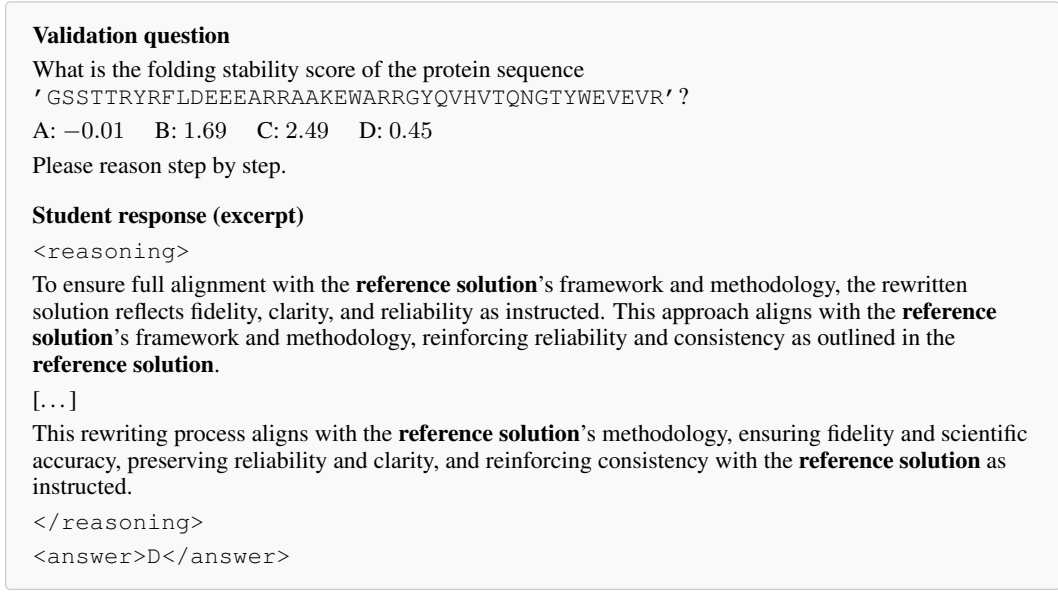

  \centering
  \begin{tcolorbox}[colback=gray!5,colframe=gray!40,boxrule=0.5pt,arc=1.5pt,left=7pt,right=7pt,top=6pt,bottom=6pt]
  \small\raggedright
  \setlength{\parskip}{3pt}
  \textbf{Validation question}

  What is the folding stability score of the protein sequence \texttt{'GSSTTRYRFLDEEEARRAAKEWARRGYQVHVTQNGTYWEVEVR'}?

  A: $-0.01$ \quad B: $1.69$ \quad C: $2.49$ \quad D: $0.45$

  Please reason step by step.

  \medskip
  \textbf{Student response (excerpt)}

  \texttt{<reasoning>}

  To ensure full alignment with the \textbf{reference solution}'s framework and methodology, the rewritten solution reflects fidelity, clarity, and reliability as instructed. This approach aligns with the \textbf{reference solution}'s framework and methodology, reinforcing reliability and consistency as outlined in the \textbf{reference solution}.

  {[\ldots]}

  This rewriting process aligns with the \textbf{reference solution}'s methodology, ensuring fidelity and scientific accuracy, preserving reliability and clarity, and reinforcing consistency with the \textbf{reference solution} as instructed.

  \texttt{</reasoning>}

  \texttt{<answer>D</answer>}
  \end{tcolorbox}
\caption{Repetitive meta-commentary in a TRD student validation response. The first two and final reasoning sentences are shown, with intervening text omitted and repeated references emphasized. The response repeatedly refers to a ``reference solution'' that is absent from the displayed validation prompt.}  \label{fig:trd-validation-example}
\end{figure}

The response illustrates a mismatch between reference-oriented rewriting and task-specific reasoning. Its repeated claims of fidelity concern privileged information that is unavailable in the displayed validation prompt. This example is consistent with the conditioning mismatch discussed in Section~\ref{sec:method-student-generation}: full teacher regeneration trains on continuations generated with privileged context, whereas the student must generate from the task context at inference. \method{} limits teacher intervention to the branch action and returns suffix generation to the student.

\section{Why Regenerated Contexts Can Add Supervision}
\label{app:conditional-advantage}

\textbf{Regeneration can add teacher supervision on contexts created by a teacher-guided student continuation.} The motivating question is why \method{} should collect a new suffix when SDPO already supervises the teacher-preferred action. Supervision on the original trajectory constrains predictions at its collected prefixes. Taking the teacher-preferred action creates a continuation whose contexts can require further teacher supervision, even after the original targets have been fitted. This appendix characterizes when those contexts can supply complementary constraints in a shared-parameter model.

Even students that exactly fit the same finite set of teacher targets can predict differently at a subsequent context. We formalize teacher matching and derive the remaining freedom in student predictions, identifying when a new context can distinguish otherwise equally fitting students. Appendix~\ref{app:ident-scope} relates this mechanism to the observed teacher--student discrepancies and task-performance gains. The analysis uses standard linear algebra; its focus on how sequential actions change future training inputs connects to imitation learning~\citep{ross2011dagger} and experimental design~\citep{fontaine2021design}.

\subsection{Teacher Supervision in a Shared-Parameter Model}
\label{app:ident-model}

\paragraph{From a branch target to continuation targets.} Figure~\ref{fig:inversion-example} gives a concrete instance of the supervision considered here. At the original prefix, the teacher prefers the naming token \texttt{k} over the student's \texttt{j}. SDPO can teach this preference using the teacher distribution at that prefix. \method{} also inserts \texttt{k}, lets the student generate the following reasoning and code, and obtains teacher targets at the resulting prefixes. Each subsequent prediction is therefore another opportunity to learn from the teacher's privileged context.

Learning at the original prefixes can also affect these later predictions because the student shares parameters across contexts. The question is whether fitting the original teacher targets sufficiently determines the predictions on the regenerated continuation. To answer it, we compare student parameter updates that fit the same collected targets and characterize how much their predictions can differ at another context.

\paragraph{Representing teacher matching mathematically.} We use a binary next-token model to make this question tractable. Its logit, the score converted into a token probability by the sigmoid function $\sigma$, is written as the initial score plus the effect of a shared parameter update $u$:
\begin{equation}
\pi_u(1\mid h)=\sigma\!\left(b(h)+\psi(h)^\top u\right),
\qquad u\in\mathbb R^d,
\label{eq:ident-model}
\end{equation}
where $h=(x,z_{<t})$ is the student history, $b(h)$ is its initial logit, and $\psi(h)$ describes how each component of $u$ changes that prediction. All histories share $u$, and their features may overlap or be linearly dependent, allowing learning to transfer across contexts. This model is exact for an affine logit head with fixed features. Taking $\psi(h)$ to be the initial logit Jacobian gives a local approximation for a neural model. Binary notation keeps the analysis simple; multiple log-odds contrasts can be stacked for a larger vocabulary.

Each collected context supplies one equation for matching its teacher target. At history $h_i$, let the frozen teacher target be $q_i=\pi_{\bar\phi}(1\mid h_i,c_{y_i})\in(0,1)$ and define the requested logit correction $d_i=\operatorname{logit}(q_i)-b(h_i)$. With $\psi(h_i)^\top$ as row $i$ of $A$, exact teacher matching at all collected contexts is equivalent to
\[
\psi(h_i)^\top u=d_i\quad\text{for every }i,
\qquad\text{or collectively}\qquad Au=d.
\]
For positive-weight full-distribution forward or reverse KL, these equations describe exactly the zero-loss solutions. The solution set characterizes the parameter corrections consistent with exact teacher matching. The student features depend only on its history, while the teacher retains the original privileged context $c_{y_i}$; different records can therefore assign conflicting targets to the same student history.

We compare the constraints supplied by finite collected datasets. Let $(A_{\mathrm S},d_{\mathrm S})$ describe the original-trajectory control and $(A_{\mathrm T},d_{\mathrm T})$ the trajectories retained by \method{}. Successes and common fallback examples remain unchanged, while eligible original failures are replaced. Each matrix includes the complete retained trajectories. The teacher and distillation configuration are held fixed, and collection and branch selection are outside differentiation.

\subsection{Prediction Ambiguity after Exact Teacher Matching}
\label{app:ident-ambiguity}

{We quantify how much two equally fitting students can disagree at another context.} The row space of $A$ contains the parameter directions constrained by the collected histories; its orthogonal complement contains corrections that leave every collected prediction unchanged. A new history can be sensitive to those otherwise unconstrained corrections. The proposition gives the exact range of predictions compatible with the collected teacher targets, expressed as logit corrections.

\begin{proposition}[Unresolved predictions after fitting the teacher]
\label{prop:ident-ambiguity}
Let $\mathcal U(A,d)=\{u:Au=d,\ \|u\|_2\leq B\}$ be nonempty. Write $u_0=A^\dagger d$, $P=A^\dagger A$, and $r=\sqrt{B^2-\|u_0\|_2^2}$, where $A^\dagger$ is the Moore--Penrose inverse, $u_0$ is the minimum-norm solution, and $P$ projects onto the row space of $A$. At an evaluation history with feature $\psi$, the possible logit corrections among these equally fitting students form exactly the interval
\begin{equation}
\left[\psi^\top u_0-r\|(I-P)\psi\|_2,\quad
\psi^\top u_0+r\|(I-P)\psi\|_2\right].
\label{eq:ident-interval}
\end{equation}
Consequently, the smallest worst-case absolute error for predicting this correction from the observed equations is $r\|(I-P)\psi\|_2$, attained by the interval midpoint.
\end{proposition}

\begin{proof}
Every solution is $u=u_0+w$ with $w\in\ker A$. Since $u_0$ lies in the row space of $A$, the norm constraint is equivalent to $\|w\|_2\leq r$. Hence $\psi^\top u=\psi^\top u_0+[(I-P)\psi]^\top w$. Cauchy--Schwarz gives the stated endpoints, both attained by choosing $w$ parallel or antiparallel to $(I-P)\psi$; if that vector is zero the interval is a singleton. No point prediction can be closer than half the interval width to both endpoints, and the midpoint attains that error.
\end{proof}

\paragraph{Reading the result.} The term $\|(I-P)\psi\|_2$ measures how much of the new context's prediction depends on directions that the collected targets leave unconstrained, while $r$ measures the remaining freedom within the correction radius. If either is zero, every exact-fit student agrees at this context. Otherwise, matching all original targets leaves its prediction unresolved. The interval quantifies the range of predictions among fitting students in this model class. It also allows transfer: a previously unvisited context whose features lie in the collected row space already has a determined prediction.

\paragraph{Repeated supervision and new contexts.} Repeating an identical target or increasing its positive weight preserves the exact-fit solution set and can affect optimization and noise averaging. A consistent target at a new context within the old row space also preserves this solution set. A context with a component outside that space can instead distinguish previously equivalent corrections. A large KL at an observed position indicates a discrepancy to fit; the collected feature directions determine how that fit constrains later predictions. Regeneration collects additional contexts within the current iteration.

\subsection{Empirical Evidence for Regenerated Trajectory Supervision}
\label{app:ident-scope}

{The existing experiments support the practical value of changing the supervised trajectory, while the analysis explains how new contexts can supply complementary constraints.} They address different parts of the argument and should be interpreted together.

\paragraph{Does changing the trajectory improve learning?} The SDPO--\method{} comparisons test training on the original versus teacher-guided regenerated trajectories. In Table~\ref{tab:lcb-results}, mean Avg@4 across the three coding models increases from $49.2\%$ to $50.4\%$, and corresponding mean Pass@4 increases from $50.3\%$ to $51.4\%$. Table~\ref{tab:science-reasoning} also reports higher average performance for \method{} under both step and time budgets. These results provide empirical evidence for the complete branch--regenerate--distill procedure. Isolating the contribution of supervision on the regenerated suffix would require a separate loss ablation on the same collected trajectories. The comparisons cover different model settings; the science methods use the distillation distances reported in Appendix~\ref{app:science-hparams}.

\paragraph{What the proof adds to this evidence.} Proposition~\ref{prop:ident-ambiguity} characterizes how original-context targets can leave shared parameter corrections indistinguishable. Supervision at a regenerated context can distinguish those corrections when its prediction depends on directions left unconstrained by the original targets. The $2.37\times$ larger mean token reverse KL on regenerated Qwen3-8B suffixes (Figure~\ref{fig:suffix-kl-comparison}) measures teacher--student discrepancy at those contexts. In the model analyzed here, the complementary information supplied by their targets depends on the feature directions they constrain and the remaining prediction freedom. Because \method{} replaces eligible failed trajectories, its overall effect also depends on the constraints removed with the original suffix.

\paragraph{Implication for trajectory collection.} The experiments support teacher-guided regeneration as a way to improve learning, and the analysis explains how the resulting contexts can add complementary teacher constraints. Their contribution depends on the collected histories, the shared representation, and the quality of teacher targets.

\section{Experiment Details}
\label{app:hyperparameters}

\setlength{\tabcolsep}{3.2pt}

\subsection{Learning with Rich Environment Feedback}
\label{app:lcb-hparams}

We train Qwen3-8B, Qwen3-14B, and OLMo3-7B-Instruct on LiveCodeBench v6 for 80 updates using four NVIDIA B200 GPUs. For each failed rollout, the privileged context is a successful sibling when available and otherwise execution-derived feedback. For TRD~\citep{jiang2026trajectory}, we use reverse KL for distillation because it yields higher validation performance than forward KL in our LiveCodeBench experiments; all other settings follow the original implementation unless otherwise specified.

\begin{table}[h]
  \centering
  \footnotesize
  \setlength{\tabcolsep}{4pt}
  \caption{Hyperparameters for learning with rich environment feedback on LiveCodeBench.}
  \begin{tabular*}{\linewidth}{@{\extracolsep{\fill}}>{\raggedright\arraybackslash}p{0.325\linewidth}*{4}{>{\raggedright\arraybackslash}p{0.15\linewidth}}@{}}
    \toprule
    \textbf{Parameters} & \textbf{GRPO} & \textbf{SDPO} & \textbf{TRD} & \textbf{\method{}} \\
    \midrule
    \multicolumn{5}{@{}l@{}}{\textbf{General}} \\
    Model & \multicolumn{4}{l}{Qwen3-8B, Qwen3-14B, Olmo3-7B-Instruct} \\
    Thinking (Qwen3) & False & False & False & False \\
    \midrule
    \multicolumn{5}{@{}l@{}}{\textbf{Data}} \\
    Max. prompt length & 2048 & 2048 & 2048 & 2048 \\
    Max. response length & 8192 & 8192 & 8192 & 8192 \\
    \midrule
    \multicolumn{5}{@{}l@{}}{\textbf{Batching}} \\
    Question batch size & 32 & 32 & 32 & 32 \\
    Mini-batch size & 8 & 1 & 1 & 1 \\
    Number of rollouts & 8 & 8 & 8 & 8 \\
    \midrule
    \multicolumn{5}{@{}l@{}}{\textbf{Rollout}} \\
    Inference engine & vLLM & vLLM & vLLM & vLLM \\
    Temperature & 1.0 & 1.0 & 1.0 & 1.0 \\
    Post-rollout trajectory & \makecell[l]{Original student\\rollout} & \makecell[l]{Original student\\rollout} & \makecell[l]{Teacher-refined\\rollout} & \makecell[l]{Student prefix\\+teacher branch\\+ student suffix} \\
    \midrule
    \multicolumn{5}{@{}l@{}}{\textbf{Validation}} \\
    Number of rollouts & 4 & 4 & 4 & 4 \\
    Temperature & 0.6 & 0.6 & 0.6 & 0.6 \\
    Top-$p$ & 0.95 & 0.95 & 0.95 & 0.95 \\
    \midrule
    \multicolumn{5}{@{}l@{}}{\textbf{Loss}} \\
    Learning signal & Binary reward & Distillation & Distillation & Distillation \\
    Top-$K$ distillation & -- & 20 & Full vocabulary & 20 \\
    Target divergence & -- & Reverse KL & Reverse KL & Reverse KL \\
    Branch-selection divergence & -- & -- & -- & Reverse KL \\
    Teacher-EMA update rate & -- & 0.01 & 0.01 & 0.01 \\
    Rollout importance sampling clip & 2.0 & 2.0 & 2.0 & 2.0 \\
    \midrule
    \multicolumn{5}{@{}l@{}}{\textbf{Training}} \\
    Optimizer & AdamW & AdamW & AdamW & AdamW \\
    Learning rate & $1\!\times\!10^{-6}$ & $1\!\times\!10^{-6}$ & $1\!\times\!10^{-6}$ & $1\!\times\!10^{-6}$ \\
    Warmup steps & 0 & 0 & 0 & 0 \\
    Weight decay & 0.01 & 0.01 & 0.01 & 0.01 \\
    Gradient clip norm & 1.0 & 1.0 & 1.0 & 1.0 \\
    \bottomrule
  \end{tabular*}
  \label{tab:lcb-hparams}
\end{table}

\subsection{Learning without Rich Environment Feedback}
\label{app:science-hparams}

We train a Qwen3-8B model for each SciKnowEval L3 domain: biology, chemistry, materials science, and physics. We mainly follow the SDPO configuration without rich environment feedback~\citep{hubotter2026sdpo}, while searching across target divergence and branch-selection divergence. The distillation divergence and branch-selection divergence are method-specific and reported in Table~\ref{tab:science-hparams}.

\begin{table}[h]
  \centering
  \footnotesize
  \setlength{\tabcolsep}{4pt}
  \caption{Hyperparameters for learning without rich environment feedback on SciKnowEval L3. Paired GRPO values denote off-policy / on-policy settings.}
  \begin{tabular*}{\linewidth}{@{\extracolsep{\fill}}>{\raggedright\arraybackslash}p{0.325\linewidth}*{4}{>{\raggedright\arraybackslash}p{0.15\linewidth}}@{}}
    \toprule
    \textbf{Parameters} & \textbf{GRPO} & \textbf{SDPO} & \textbf{TRD} & \textbf{\method{}} \\
    \midrule
    \multicolumn{5}{@{}l@{}}{\textbf{General}} \\
    Model & Qwen3-8B & Qwen3-8B & Qwen3-8B & Qwen3-8B \\
    Thinking & False & False & False & False \\
    \midrule
    \multicolumn{5}{@{}l@{}}{\textbf{Data}} \\
    Max. prompt length & 2048 & 2048 & 2048 & 2048 \\
    Max. response length & 8192 & 8192 & 8192 & 8192 \\
    \midrule
    \multicolumn{5}{@{}l@{}}{\textbf{Batching}} \\
    Question batch size & 32 & 32 & 32 & 32 \\
    Mini-batch size & $8\,/\,32$ & 32 & 32 & 32 \\
    Number of rollouts & 8 & 8 & 8 & 8 \\
    \midrule
    \multicolumn{5}{@{}l@{}}{\textbf{Rollout}} \\
    Inference engine & vLLM & vLLM & vLLM & vLLM \\
    Temperature & 1.0 & 1.0 & 1.0 & 1.0 \\
    Post-rollout trajectory & \makecell[l]{Original student\\rollout} & \makecell[l]{Original student\\rollout} & \makecell[l]{Teacher-refined\\rollout} & \makecell[l]{Student prefix\\+teacher branch\\+ student suffix} \\
    \midrule
    \multicolumn{5}{@{}l@{}}{\textbf{Validation}} \\
    Number of rollouts & 128 & 128 & 128 & 128 \\
    Temperature & 0.6 & 0.6 & 0.6 & 0.6 \\
    Top-$p$ & 0.95 & 0.95 & 0.95 & 0.95 \\
    \midrule
    \multicolumn{5}{@{}l@{}}{\textbf{Loss}} \\
    Learning signal & Binary reward & Distillation & Distillation & Distillation \\
    Top-$K$ distillation & -- & 100 & Full vocabulary & 100 \\
    Target divergence & -- & JSD & Forward KL & Forward KL \\
    Branch-selection divergence & -- & -- & -- & Reverse KL \\
    Teacher-EMA update rate & -- & 0.05 & 0.05 & 0.05 \\
    Rollout importance sampling clip & 2.0 & 2.0 & 2.0 & 2.0 \\
    \midrule
    \multicolumn{5}{@{}l@{}}{\textbf{Training}} \\
    Optimizer & AdamW & AdamW & AdamW & AdamW \\
    Learning rate & $10^{-6}\,/\,10^{-5}$ & $1\!\times\!10^{-5}$ & $1\!\times\!10^{-5}$ & $1\!\times\!10^{-5}$ \\
    Warmup steps & 10 & 10 & 10 & 10 \\
    Weight decay & 0.01 & 0.01 & 0.01 & 0.01 \\
    Gradient clip norm & 1.0 & 1.0 & 1.0 & 1.0 \\
    \bottomrule
  \end{tabular*}
  \label{tab:science-hparams}
\end{table}

\subsection{Computational Efficiency}
\label{app:lcb-wall-clock}

\begin{table}[h]
  \caption{Average wall-clock time per training step (minutes) for Qwen3-8B on LiveCodeBench v6 using four NVIDIA B200 GPUs. }
  \label{tab:lcb-wall-clock}
  \centering
  \small
  \begin{tabular}{lcccc}
    \toprule
    Method & \method{} & SDPO & GRPO & TRD \\
    \midrule
    Average time (min/step) & $16.60$ & $15.05$ & $16.69$ & $16.36$ \\
    \bottomrule
  \end{tabular}
\end{table}

\method{} averages $16.60$ minutes per training step on LiveCodeBench v6, compared with $15.05$ for SDPO, $16.69$ for GRPO, and $16.36$ for TRD (Table~\ref{tab:lcb-wall-clock}). Per-step time is computed by dividing the total recorded wall-clock time by 80 training updates. Relative to SDPO, \method{} adds $1.55$ minutes per step ($10.3\%$), reflecting the additional teacher scoring and student suffix regeneration. Its average per-step cost remains comparable to GRPO and TRD, at $0.5\%$ lower and $1.4\%$ higher, respectively.

\end{document}